\documentclass[11pt]{article}
\usepackage[margin=1in]{geometry}
\usepackage{amsmath,amssymb,amsthm,mathtools}
\usepackage[T1]{fontenc}
\usepackage{lmodern}
\usepackage{microtype}
\usepackage{xcolor}
\usepackage{hyperref}
\usepackage{enumitem}
\hypersetup{colorlinks=true,linkcolor=black,citecolor=blue,urlcolor=blue}

\newtheorem{proposition}{Proposition}
\newtheorem{remark}{Remark}

\title{On the ``Causality'' Step in Policy Gradient Derivations: \\
A Pedagogical Reconciliation of Full Return and Reward-to-Go}
\author{Nima H. Siboni \\
Juna.ai \\
Kastanienallee 32, 10435 Berlin \\
\href{mailto:nima@juna.ai}{nima@juna.ai}}
\date{}

\begin{document}
\maketitle

\begin{abstract}
In introductory presentations of policy gradients, one often derives the REINFORCE estimator using the full trajectory return and then states, by ``causality,'' that the full return may be replaced by the reward-to-go. Although this statement is correct, it is frequently presented at a level of rigor that leaves unclear where the past-reward terms disappear. This short paper isolates that step and gives a mathematically explicit derivation based on prefix trajectory distributions and the score-function identity. The resulting account does not change the estimator. Its contribution is conceptual: instead of presenting reward-to-go as a post hoc unbiased replacement for full return, it shows that reward-to-go arises directly once the objective is decomposed over prefix trajectories. In this formulation, the usual causality argument is recovered as a corollary of the derivation rather than as an additional heuristic principle.
\end{abstract}

\section{Introduction}
The policy-gradient method is commonly introduced through the log-derivative trick applied to the expected trajectory return. In pedagogical treatments, one first arrives at an estimator involving the \emph{full} return and then improves it by arguing that an action at time $t$ cannot influence rewards obtained before time $t$, so the estimator may instead use the \emph{reward-to-go}. This move appears, for example, in Berkeley's deep reinforcement learning lecture on policy gradients \cite{berkeley}. It is intuitively compelling and algorithmically important, since replacing full return by reward-to-go reduces variance without changing the expectation.

At the same time, the way this step is usually presented can leave a mathematically careful reader dissatisfied. The issue is not whether the conclusion is true, but rather how it is justified: at what line in the derivation do the terms involving rewards from times $< t$ vanish, and what precise probabilistic identity is being used?

The aim of this paper is modest. We do not present a new theorem or a new estimator. Instead, we clarify a difference in \emph{explanatory structure}. In the standard REINFORCE narrative, one proceeds schematically as
\[
\text{full return} \;\longrightarrow\; \text{causality argument} \;\longrightarrow\; \text{reward-to-go}.
\]
That is, reward-to-go appears as an unbiased substitution made \emph{after} the full-return estimator has already been written down \cite{williams1992,spinningup,berkeley}. By contrast, the derivation given here proceeds as
\[
\text{prefix trajectory decomposition} \;\longrightarrow\; \text{score expansion} \;\longrightarrow\; \text{reward-to-go},
\]
and only \emph{then} identifies the usual causality statement as the corresponding zero-expectation identity. The difference is therefore not algorithmic but conceptual: in the present account, reward-to-go is not introduced as a correction to a previously derived estimator, but as the quantity canonically paired with each score term once the objective is written at the level of partial trajectories.

This distinction also explains how the lecture-note style and the theorem-oriented literature fit together. In the former, one starts from a trajectory-level expression and must later explain why only future rewards remain attached to $\nabla_\theta \log \pi_\theta(a_t\mid s_t)$. In the latter, the policy gradient is written directly in terms of an action-value object, so the future-only dependence is already built in \cite{sutton1999}. Our main point is that these are not competing viewpoints. The ``causality'' step is simply a compressed way of referring to a precise probabilistic identity, and the prefix-trajectory formulation makes that identity visible from the beginning.

\section{Setup and notation}
Consider a finite-horizon Markov decision process with horizon $T$, initial distribution $p(s_1)$, transition kernel $p(s_{t+1}\mid s_t,a_t)$, and differentiable stochastic policy $\pi_\theta(a_t\mid s_t)$. Let the objective be
\begin{equation}
J(\theta)=\mathbb E_{\tau\sim p_\theta(\tau)}\!\left[\sum_{t=1}^T r(s_t,a_t)\right],
\label{eq:Jfull}
\end{equation}
where the full trajectory density is
\begin{equation}
p_\theta(\tau)=p(s_1)\prod_{t=1}^{T} \pi_\theta(a_t\mid s_t)\prod_{t=1}^{T-1} p(s_{t+1}\mid s_t,a_t).
\label{eq:fulltraj}
\end{equation}
For each $t$, define the prefix trajectory
\[
\tau_t=(s_1,a_1,\dots,s_t,a_t),
\]
with density
\begin{equation}
p_\theta(\tau_t)=p(s_1)\prod_{i=1}^{t} \pi_\theta(a_i\mid s_i)\prod_{i=1}^{t-1} p(s_{i+1}\mid s_i,a_i).
\label{eq:prefixdensity}
\end{equation}
Writing $r_t:=r(s_t,a_t)$, we may express the performance objective as
\begin{equation}
J(\theta)=\sum_{t=1}^T \int p_\theta(\tau_t)\,r_t\,d\tau_t.
\label{eq:Jprefix}
\end{equation}
This form emphasizes a simple but important fact: the reward incurred at time $t$ depends only on the prefix up to time $t$. It is precisely here that the present derivation departs from the standard narrative. Instead of first differentiating a single expectation involving the \emph{total} return and later discarding past rewards, we decompose the objective into time-indexed contributions from the outset. This is the structural move that will make reward-to-go appear without any separate subtraction argument.

For comparison, the corresponding expression in the common narrative, written at the same level as \eqref{eq:Jprefix}, would be
\begin{equation}
J(\theta)=\sum_{t=1}^T \int p_\theta(\tau)\,r_t\,d\tau.
\label{eq:commonJform}
\end{equation}
Equivalently,
\[
J(\theta)=\int p_\theta(\tau)\left(\sum_{t=1}^T r_t\right)d\tau.
\]
Thus, the precise mathematical difference is already visible at the level of the objective: \eqref{eq:Jprefix} writes each reward $r_t$ against the \emph{prefix} measure $p_\theta(\tau_t)$, whereas \eqref{eq:commonJform} writes each reward against the \emph{full} trajectory measure $p_\theta(\tau)$.

\section{Where the reward-to-go comes from}
Assume that differentiation may be interchanged with integration in \eqref{eq:Jprefix}. Then
\begin{align}
\nabla_\theta J(\theta)
&=\sum_{t=1}^T \int \nabla_\theta p_\theta(\tau_t)\,r_t\,d\tau_t \\
&=\sum_{t=1}^T \int p_\theta(\tau_t)\,\nabla_\theta \log p_\theta(\tau_t)\,r_t\,d\tau_t.
\label{eq:scoreprefix}
\end{align}
Up to this point, the manipulation is a standard application of the score-function identity. The divergence from the usual REINFORCE presentation lies in \emph{what} the identity is being applied to. In the standard route, one applies it to the full trajectory measure and obtains an estimator with the total return multiplying every score term. Here it is applied separately to each prefix measure $p_\theta(\tau_t)$, so each immediate reward $r_t$ is tied only to the prefix on which it actually depends.

Since the environment dynamics and initial-state distribution do not depend on $\theta$, the only $\theta$-dependence in \eqref{eq:prefixdensity} comes from the policy factors. Hence
\begin{equation}
\nabla_\theta \log p_\theta(\tau_t)=\sum_{i=1}^t \nabla_\theta \log \pi_\theta(a_i\mid s_i).
\label{eq:prefixscore}
\end{equation}
Substituting \eqref{eq:prefixscore} into \eqref{eq:scoreprefix} gives
\[
\nabla_\theta J(\theta)
=\sum_{t=1}^T \int p_\theta(\tau_t)\left(\sum_{i=1}^t \nabla_\theta \log \pi_\theta(a_i\mid s_i)\right)r_t\,d\tau_t.
\]
Now exchange the order of summation:
\begin{equation}
\nabla_\theta J(\theta)
=\sum_{j=1}^T\sum_{t=j}^T \int p_\theta(\tau_t)\,\nabla_\theta\log\pi_\theta(a_j\mid s_j)\,r_t\,d\tau_t.
\label{eq:swappedsums}
\end{equation}
This is the first line at which the conceptual difference becomes fully visible. In the standard derivation, one would already have the full return multiplying $\nabla_\theta \log \pi_\theta(a_j\mid s_j)$ and would now need to prove that the rewards from times $t<j$ can be deleted. Here no such deletion is required. The index restriction $t\geq j$ appears automatically because the score term at time $j$ is present only in prefixes of length at least $j$. In other words, reward-to-go emerges by \emph{construction}, not by later cancellation.

To make the consequence explicit, fix $j$. Using the factorization
\[
p_\theta(\tau_t)=p_\theta(\tau_j)\,p_\theta(\tau_{j+1:t}\mid \tau_j),
\]
we may integrate out the future variables beyond time $j$:
\begin{align*}
&\sum_{t=j}^T \int p_\theta(\tau_t)\,\nabla_\theta\log\pi_\theta(a_j\mid s_j)\,r_t\,d\tau_t \\
&\qquad = \int p_\theta(\tau_j)\,\nabla_\theta\log\pi_\theta(a_j\mid s_j)
\left(\sum_{t=j}^T \mathbb E[r_t\mid \tau_j]\right)d\tau_j.
\end{align*}
Again, the present route differs from the standard classroom story in emphasis. One is not starting from a term involving all rewards and then showing that earlier rewards are harmless. Rather, after regrouping by a fixed score term, one finds that only rewards indexed by $t\geq j$ are available to be paired with that score term in the first place. The object that remains is therefore the future return from time $j$ onward.

If one keeps the sampled future rewards rather than writing the conditional expectation explicitly, the quantity multiplying the score term is exactly the realized reward-to-go
\begin{equation}
R_j:=\sum_{t=j}^T r_t.
\label{eq:rtgdef}
\end{equation}
Thus one obtains the familiar reward-to-go form
\begin{equation}
\nabla_\theta J(\theta)=\sum_{j=1}^T \mathbb E\big[\nabla_\theta\log\pi_\theta(a_j\mid s_j)\,R_j\big].
\label{eq:rtgpg}
\end{equation}
The formula is standard; what is nonstandard in the present note is only the logical order by which it is reached.

\section{The zero-expectation identity behind ``causality''}
A second, equivalent way to explain the same step is to show directly why past rewards contribute nothing in expectation.

\begin{proposition}
For any $t<j$,
\begin{equation}
\mathbb E\big[\nabla_\theta\log\pi_\theta(a_j\mid s_j)\,r_t\big]=0.
\label{eq:zeroterm}
\end{equation}
\end{proposition}

\begin{proof}
Condition on the history up to time $j$, excluding the random draw of $a_j$. The reward $r_t$ is then measurable with respect to the past and does not depend on $a_j$, so
\begin{align*}
\mathbb E\big[\nabla_\theta\log\pi_\theta(a_j\mid s_j)\,r_t\mid s_j,\text{past}\big]
&=r_t\int \pi_\theta(a_j\mid s_j)\,\nabla_\theta\log\pi_\theta(a_j\mid s_j)\,da_j \\
&=r_t\,\nabla_\theta\int \pi_\theta(a_j\mid s_j)\,da_j \\
&=r_t\,\nabla_\theta 1 \\
&=0.
\end{align*}
Taking expectations proves \eqref{eq:zeroterm}.
\end{proof}

This proposition is the precise identity hidden behind the sentence ``the current action cannot affect past rewards.'' It also explains how the present derivation relates to the standard one. If one begins with the full-return REINFORCE expression, the role of Proposition~1 is to justify \emph{removing} the past-reward terms. In the prefix-based derivation above, those terms never appear after regrouping. Thus the usual causality argument is not contradicted; it is recovered as the statement that the extra terms introduced by the full-return route would in fact have expectation zero.

\section{Relation to the policy gradient theorem}
The same conclusion appears in a more structural way in the policy gradient theorem. Sutton et al.\ show that the policy gradient can be written in terms of the action-value function $Q^{\pi}(s,a)$ rather than the full return \cite{sutton1999}. From that perspective, future dependence is present from the outset: $Q^{\pi}(s_t,a_t)$ is already the expected return from time $t$ onward. The reward-to-go estimator is then simply the Monte Carlo sample counterpart of that object.

This is why the theorem-oriented literature can seem cleaner on this point. It avoids the need to explicitly cancel past rewards because they have already been excluded by the definition of the value object being used. By contrast, lecture derivations that begin from the full return keep all rewards in view initially and must subsequently explain why only future rewards remain attached to each score term. The prefix-trajectory derivation developed here should be understood as a bridge between these two styles. It still begins from a trajectory-level decomposition, as lecture notes often do, but it organizes the calculation in a way that makes the future-only structure explicit before any appeal to causality.

\section{Pedagogical conclusion}
The lecture-note phrase ``by causality, use reward-to-go'' should therefore be read as an abbreviation, not as a separate heuristic principle. What it abbreviates is one of two equivalent arguments:
\begin{enumerate}[label=(\alph*)]
    \item a regrouping argument over prefix trajectory distributions, leading from \eqref{eq:scoreprefix} to \eqref{eq:rtgpg}; or
    \item a zero-expectation argument showing that every term coupling $\nabla_\theta\log\pi_\theta(a_j\mid s_j)$ with a reward $r_t$ for $t<j$ vanishes.
\end{enumerate}
Seen this way, the pedagogical derivation and the theorem-oriented derivation are fully reconciled. The standard REINFORCE narrative presents reward-to-go as an unbiased replacement for total return; the prefix-trajectory narrative shows that reward-to-go is already the quantity canonically associated with each score term once the objective is decomposed at the correct level of granularity.

\begin{remark}
From a teaching perspective, the compressed ``causality'' slogan is probably useful on first encounter, since it highlights the direction of influence in a trajectory. But for a reader seeking a fully rigorous derivation, it is worth making explicit that the real engine is the score-function identity together with the factorization of prefix trajectory measures. The pedagogical gain of the present reformulation is exactly that it makes clear \emph{where} the derivation diverges from the usual REINFORCE story and \emph{why} reward-to-go appears without an extra heuristic step.
\end{remark}

\end{document}